\documentclass[letterpaper]{article} 
\usepackage{aaai24}  
\usepackage{times}  
\usepackage{helvet}  
\usepackage{courier}  
\usepackage[hyphens]{url}  
\usepackage{graphicx} 
\usepackage{natbib}  
\usepackage{caption} 
\usepackage{algorithm}
\usepackage{algorithmic}
\usepackage{textcomp}
\usepackage[dvipsnames]{xcolor}
\usepackage{graphicx}
\usepackage{subcaption}
\usepackage{subfiles}
\usepackage{multirow}
\usepackage{makecell}
\usepackage{multicol}
\usepackage{epsfig}
\usepackage{pifont}
\usepackage{amsmath,amssymb,amsfonts}
\usepackage{amsthm}
\usepackage{appendix}
\usepackage{nicefrac}
\usepackage{microtype}
\usepackage{booktabs}
\newtheorem{theorem}{Theorem}
\newtheorem{corollary}{Corollary}
\newtheorem{lemma}[theorem]{Lemma}

\newtheorem{proposition}{Proposition}

\newtheorem{definition}{Definition}

\usepackage{newfloat}
\usepackage{listings}
\DeclareCaptionStyle{ruled}{labelfont=normalfont,labelsep=colon,strut=off} 
\floatstyle{ruled}
\newfloat{listing}{tb}{lst}{}
\floatname{listing}{Listing}
\title{Brave: Byzantine-Resilient and Privacy-Preserving\\ Peer-to-Peer Federated Learning}
\author{
    Zhangchen Xu\equalcontrib\textsuperscript{\rm 1}, Fengqing Jiang\equalcontrib\textsuperscript{\rm 1}, Luyao Niu\textsuperscript{\rm 1}, Jinyuan Jia\textsuperscript{\rm 2}, Radha Poovendran\textsuperscript{\rm 1}
}
\affiliations{
    \textsuperscript{\rm 1}University of Washington\\
    \textsuperscript{\rm 2}The Pennsylvania State University\\
    \{zxu9,fqjiang,luyaoniu,rp3\}@uw.edu, jinyuan@psu.edu

}

\begin{document}

\maketitle

\begin{abstract}
Federated learning (FL) enables multiple participants to train a global machine learning model without sharing their private training data.
Peer-to-peer (P2P) FL advances existing centralized FL paradigms by eliminating the server that aggregates local models from participants and then updates the global model.
However, P2P FL is vulnerable to (i) \emph{honest-but-curious} participants whose objective is to infer private training data of other participants, and (ii) \emph{Byzantine} participants who can transmit arbitrarily manipulated local models to corrupt the learning process.
P2P FL schemes that simultaneously guarantee Byzantine resilience and preserve privacy have been less studied.
In this paper, we develop \texttt{Brave}, a protocol that ensures \underline{B}yzantine \underline{R}esilience \underline{A}nd pri\underline{V}acy-pr\underline{E}serving property for P2P FL in the presence of both types of adversaries. 
We show that \texttt{Brave} preserves privacy by establishing that any honest-but-curious adversary cannot infer other participants' private data by observing their models.
We further prove that \texttt{Brave} is Byzantine-resilient, which guarantees that all benign participants converge to an identical model that deviates from a global model trained without Byzantine adversaries by a bounded distance.
We evaluate \texttt{Brave} against three state-of-the-art adversaries on a P2P FL for image classification tasks on benchmark datasets CIFAR10 and MNIST.
Our results show that the global model learned with \texttt{Brave} in the presence of adversaries achieves comparable classification accuracy to a global model trained in the absence of any adversary.
\end{abstract}

\section{Introduction}

Federated learning (FL) \cite{mcmahan2017communication,konevcny2016federated} allows multiple participants to collaboratively train a global model while avoiding sharing their private data.
In a classic centralized FL paradigm, a central server coordinates the training process among the distributed participants and aggregates their local models. However, centralized FL is not viable when the server is subject to faults or becomes untrustworthy and/or unavailable \cite{lalitha2019peer}. 

To address the high dependence on the reliability of the server in centralized FL, Peer-to-Peer (P2P) federated learning has been proposed in \cite{lalitha2019peer,roy2019braintorrent}. 
In P2P FL, the global model is updated by leveraging information exchange between individual participants, thereby eliminating the need for a central server. The P2P paradigm presents two significant advantages.
First, it eliminates the potential vulnerabilities associated with a single point of failure. 
Second, it can be deployed for decentralized applications such as vehicle-to-vehicle networks \cite{chen2021bdfl} and Internet-of-Things \cite{stefano2020Federated}.

Due to the absence of central coordination in the P2P setting, the behaviors of participants are unpredictable and uncontrollable. 
An unauthorized agent can participate in P2P FL aiming to obtain sensitive data and/ or disrupt the training process. 
Since local models are shared among participants, honest-but-curious participants may be able to infer private training data by launching a membership inference attack \cite{nasr2019comprehensive}.
Therefore, it is crucial to implement privacy-preserving techniques such as secure Multi-Party Computation (MPC) \cite{yao1982protocols} for model aggregation in P2P settings.
Besides honest-but-curious participants, Byzantine participants can send inconsistent models to different participants in P2P networks, inject biased models, or even withhold its updates \cite{tolpegin2020data-labelflipping, xie2019zeno-bitflipping, fang2020local-gaussian}.
As a consequence, Byzantine participants can degrade the global model or even cause the FL to fail.
Therefore, Byzantine-resilient model aggregation rules are needed to defend against Byzantine attacks in P2P FL. 

The current Byzantine-resilient model aggregation rules \cite{blanchard2017machine,yin2018byzantine,li2019rsa}, however, are not readily compatible with privacy-preserving techniques such as MPC.
These model aggregation rules \cite{blanchard2017machine,yin2018byzantine,li2019rsa} eliminate the compromised models from adversaries by disclosing each individual participant's information. 
Consequently, the honest-but-curious adversaries can infer private training data owned by the other participants, leveraging such information disclosure.
At present, however, P2P FL schemes that simultaneously guarantee Byzantine resilience and preserve privacy have been less investigated.

In this paper, we design a P2P FL scheme that achieves Byzantine resilience while preserving privacy.
We consider the presence of both honest-but-curious and Byzantine participants, and define three properties: \textit{information-theoretic privacy}, \textit{$\epsilon$-convergence}, and \textit{agreement}. Information-theoretic privacy ensures that no information about the participants' local models is leaked during the training process, $\epsilon$-convergence implies that the distance between global models learned with and without Byzantine participants is at most $\epsilon$, and agreement indicates the global model of all benign participants are identical.
We guarantee information-theoretic privacy by first letting each participant make a commitment of its local model, which will be `locked' and thus not editable in the future.
We then utilize MPC technique to compare and sort the participants' local models without disclosing their true values.
Each participant then invokes a trimming scheme to exclude the largest and smallest $f$ local models when updating the global model, where $f$ is the maximum number of Byzantine participants.
The main contributions of our paper are summarized as follows.

\begin{itemize}
    \item We propose \texttt{Brave}, a \underline{B}yzantine \underline{R}esilience \underline{A}nd pri\underline{V}acy-pr\underline{E}serving protocol for P2P FL.
    We prove that \texttt{Brave} ensures the local model of each participant to be information-theoretically private during the learning process.
    \item We design a privacy-preserving trimming scheme to ensure $\epsilon$-convergence in the presence of Byzantine adversaries.
    We further leverage distributed consensus to ensure agreement. We theoretically prove that \texttt{Brave} is resilient to Byzantine participants given $N>3f+2$, where $N$ is the total number of participants. 
    \item We evaluate \texttt{Brave} against three state-of-the-art adversaries on two image classification tasks. Our results show that \texttt{Brave} guarantees $\epsilon$-convergence if $N>3f+2$ holds. Furthermore, the global model trained using P2P FL that implements \texttt{Brave} achieves comparable classification accuracy to a global model learned in the absence of any adversary.
\end{itemize}

The rest of the paper is organized as follows. We present the literature review in Section \ref{sec: related work}. In Section \ref{sec: system model}, we present the system model and problem formulation. 
Section \ref{sec: protocol} describes \texttt{Brave}, a multi-stage protocol in P2P FL that is information-theoretic privacy and Byzantine resilience.
Section \ref{sec: performance evaluation} shows empirical results of \texttt{Brave} for image classification tasks on benchmark datasets including CIFAR10 and MNIST. Section \ref{sec: conclusion} concludes this paper and discusses future work.

\section{Related Work}\label{sec: related work}

In this section, we first review Byzantine-resilient and privacy-preserving solutions developed for centralized FL. We then present P2P FL schemes that guarantee Byzantine resilience and/ or preserve privacy.

\textit{Solutions in Centralized Federated Learning:}
To preserve privacy in FL, differential privacy (DP) \cite{dwork2008differential} based mechanisms were widely used to protect model parameters \cite{truex2019hybrid,wei2020federated}.
However, \cite{hitaj2017deep} pointed out that DP is not adequate for safeguarding privacy in collaborative machine learning. Additionally, DP factors could result in reduction of model accuracy \cite{bagdasaryan2019differential}. 
An alternative secure multi-party computation (MPC) \cite{yao1982protocols} based approach was adopted in \cite{bonawitz2017practical}, which presented a model aggregation method for FL using a pairwise mask to cloak the local model. 
In \cite{mugunthan2019smpai}, DP and MPC were simultaneously applied to further strengthen the privacy-preserving property.

Multiple Byzantine-resilient model aggregation rules were proposed to address the presence of Byzantine adversaries in FL.
Multi-Krum \cite{blanchard2017machine} updated the global model by selecting the gradient that minimizes the sum of squared distances to its $N-f$ closest gradients.
Similarly, \cite{yin2018byzantine} proposed two coordinate-wise model aggregation rules, trimmed mean and median. 
In trimmed mean (resp. median), the server sorted the value of each model, removed the largest and smallest $f$ values, and then computed the mean (resp. median) of the remaining values for the global model update.

To simultaneously preserve privacy and guarantee Byzantine resilience, DP and Byzantine-resilient model aggregation techniques were adopted in \cite{ma2022privacy,ma2022differentially}.
The authors of \cite{he2020secure,liu2021scalable} implemented Multi-Krum on two non-colluding honest-but-curious servers for secure aggregation, and applied private calculation to ensure privacy.
It is worth noting that \cite{liu2021scalable} implemented median as the Byzantine-resilience aggregation rule, where local models from all participants were sorted using MPC.
In \cite{velicheti2021secure}, a clustered FL framework that randomly clusters clients before filtering malicious updates was proposed, which cryptographically preserved the local model privacy and was robust to Byzantine adversaries. 
A very recent work \cite{mohamad2023sok} discussed the potential challenges in secure aggregation based on cryptographic schemes for FL.

\textit{Solutions in P2P Federated Learning:}
There are two main categories in realizing P2P FL. The first category adopted a fully decentralized framework where no coordinator existed.
In \cite{lalitha2019peer}, the participants aggregated data and updated their model by observing their one-hop neighbors. 
In \cite{roy2019braintorrent}, the participant who intended to renew its local model actively initiated an update by requesting the latest model from other participants.
Another category leveraged blockchain \cite{Li2021Blockchain,ramanan2020baffle,chen2021bdfl,shayan2021biscotti} or distributed consensus algorithms \cite{che2021decentralized,han2022defl} to synchronize the FL process and coordinate the model aggregation. 
For example, \cite{chen2021bdfl} proposed a blockchain-based P2P FL, where the public verifiable secret sharing (PVSS) scheme was implemented to preserve privacy. 
In \cite{Li2021Blockchain}, blockchain was used for global model storage and the local model update exchange, and a committee consensus mechanism was proposed to reduce the amount of consensus computing. 
In \cite{ramanan2020baffle}, smart contract was applied to coordinate the round delineation, model aggregation, and update tasks in FL. 
In \cite{che2021decentralized,han2022defl}, two distributed consensus algorithms, PBFT \cite{castro1999practical} and Hotstuff \cite{yin2019hotstuff}, were applied to coordinate the FL training process, and Multi-Krum was implemented to enable Byzantine-resilience model aggregation. 
However, these solutions did not consider privacy-preserving and Byzantine resilience simultaneously.  

Currently, Byzantine-resilient and privacy-preserving P2P FL has been less studied. The authors of \cite{shayan2021biscotti} simultaneously considered privacy-preserving property and Byzantine resilience in P2P FL.
Blockchain and PVSS were used in \cite{shayan2021biscotti} to ensure differential privacy, whereas Multi-Krum was applied to defend against Byzantine clients.
The approach in \cite{shayan2021biscotti} may inherit the performance degradation with respect to DP factors \cite{hitaj2017deep}.
Different from \cite{shayan2021biscotti}, \texttt{Brave} is information-theoretically private. 
We further demonstrate that \texttt{Brave} does not degrade the performance of global model trained by P2P FL.
Therefore, the present paper is complementary to \cite{shayan2021biscotti}.

\section{System Model and Problem Formulation}
\label{sec: system model}

\textbf{System Model.} We consider a P2P FL \cite{lalitha2019peer,roy2019braintorrent} consisting of a set of participants $\mathbb{P}=\{P_1,\ldots,P_N\}$ who aim to learn a global model ${w}$. Each participant $P_i$ owns a certain amount of private data $\mathcal{D}_i$. At iteration $t$ of P2P FL, each participant $P_i$ updates its local model ${w}_i(t)$ using gradient descent as 
\begin{equation}
\setlength{\abovedisplayskip}{3pt}
\setlength{\belowdisplayskip}{3pt}
	{w}_i(t+1)={w}(t)-\eta  {g}_i(t),
	\label{equation weight update}
\end{equation}
where ${g}_i(t) = \frac{\partial \mathcal{L}(\mathcal{D}_i;{w}(t))}{\partial {w}(t)}$ is the gradient, $\mathcal{L}(\mathcal{D}_i;{w}(t))$ is the loss function of $P_i$, and $\eta$ is the learning rate.
Participant $P_i$ then receives local models ${w}_j(t+1)$ from other participants $j\neq i$, and updates the global model as 
\begin{equation}
\setlength{\abovedisplayskip}{3pt}
\setlength{\belowdisplayskip}{3pt}
	{w}(t+1)=\frac{1}{N} \sum_{i=1}^N {w}_i(t+1).
	\label{equation directly average weight}
\end{equation}
Such a procedure repeats until a stopping criterion is met.

\textbf{Threat Model.} We consider a P2P FL framework where both \textit{passive} and \textit{Byzantine} adversaries exist and potentially overlap, with the remaining participants identified as \textit{benign}.
The adversaries are assumed to have full access to the messages they receive but are incapable of eavesdropping or intercepting the communications of others. Furthermore, these adversaries have limited computational capability to solve the discrete logarithms problem \cite{mccurley1990discrete}. This assumption is the foundation for the security of numerous public key systems and protocols \cite{lee1998efficient,joux2014past}. 
Specifically, \textbf{Passive adversaries} \cite{nasr2019comprehensive} follow the procedure of P2P FL but aim to obtain the local models from the other participants, thereby extrapolating private training data by launching MIA \cite{nasr2019comprehensive} on these local models.
\textbf{Byzantine adversaries} aim at compromising the learning performance of P2P FL by biasing the local models of other participants. In pursuit of their objectives, the Byzantine adversaries can create compromised local models, and send different local models to different participants or just remain silent in the communication process. 

\textbf{Problem Formulation.} 
Our aim of this paper is to develop a protocol such that the P2P FL guarantees both \textit{information-theoretic privacy} and \textit{Byzantine resilience}, defined as below.

\begin{definition}[Information-Theoretic Privacy]
A benign participant $P_i$'s local model ${w}_i$ is information-theoretically private if its local model $w_i$ cannot be revealed by a passive adversary by observing the message sent by $P_i$.
\end{definition}
 
\begin{definition}[Byzantine Resilience]\label{def:byzantine resilience}
    A P2P FL scheme is Byzantine-resilient if the following properties hold: 1) \emph{$\epsilon$-Convergence}:  There exists $\epsilon\geq 0$ such that $\|{w}(T)-{w}(T)^*\|\leq \epsilon$ holds with probability at least $\zeta$, where ${w}(T)^*$ is the global model obtained using Eqn. \eqref{equation weight update} and \eqref{equation directly average weight} without Byzantine adversaries, and $T$ is the iteration index when P2P FL terminates. 2) \emph{Agreement}: The global model $w(t)$ of all benign participants is identical for each iteration $t$.
\end{definition}

\textit{Remark.} Ensuring information-theoretic privacy and Byzantine resilience implies the protocol is resilient to both passive adversaries and Byzantine adversaries.

\section{Design of \texttt{Brave}}
\label{sec: protocol}
We designed \texttt{Brave}, a multi-stage protocol in P2P FL that comprises four stages: \textit{Commitment}, \textit{Privacy-preserving Comparison}, \textit{Sorting \& Trimming}, and \textit{Aggregation \& Verification}. The proposed protocol is proven to achieve information-theoretic privacy and Byzantine resilience if $N>3f+2$. 

\begin{figure}[htbp]
  \centering 
  \includegraphics[width=0.46\textwidth]{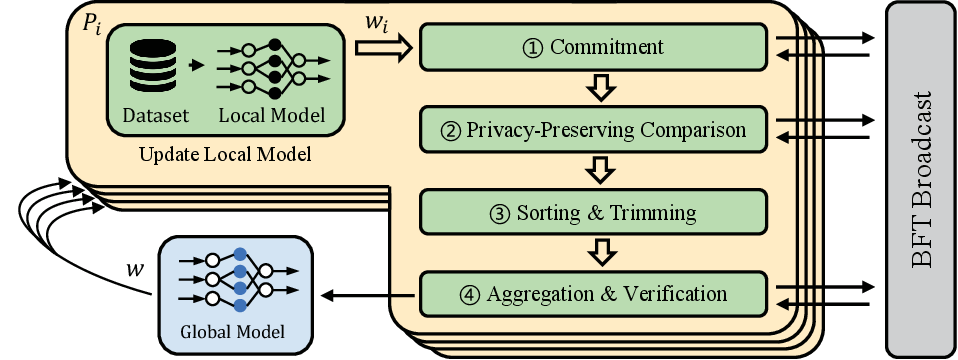}
  \caption{This figure shows the overall workflow of \texttt{Brave}.}
  \label{fig: algorithm flow} 
\end{figure}

To defend against Byzantine adversaries who might input differing models at various stages, \texttt{Brave} starts with a commitment stage (\ding{192}) after participants update their local models in each iteration. Following this, \texttt{Brave} enters a privacy-preserving comparison stage (\ding{193}), which enables the sorting of local models without revealing any information about their true values. \texttt{Brave} then incorporate trimmed mean \cite{yin2018byzantine} in the sorting \& trimming stage (\ding{194}) to remove the outliers possibly introduced by Byzantine adversaries. In the last aggregation \& verification stage (\ding{195}), MPC \cite{bonawitz2017practical} is performed to aggregate models while preserving privacy. This stage also verifies the consistency of the aggregated model with the commitment. Fig. \ref{fig: communication structure} illustrates the schematic message flow of \texttt{Brave} when applied to a P2P FL with six participants.

\textbf{Notations.} We use $\mathbb{Z}_q$ to denote the set of non-negative integers that are no larger than $q$.
We denote $\mathbb{Z}_q^m$ as the set of $m$-dimensional vectors with each entry taking values from $\mathbb{Z}_q$. 
We denote set $\{1,\ldots,m\}$ as $[m]$.

\begin{figure*}[htbp]
  \centering 
  \includegraphics[width=0.94\textwidth]{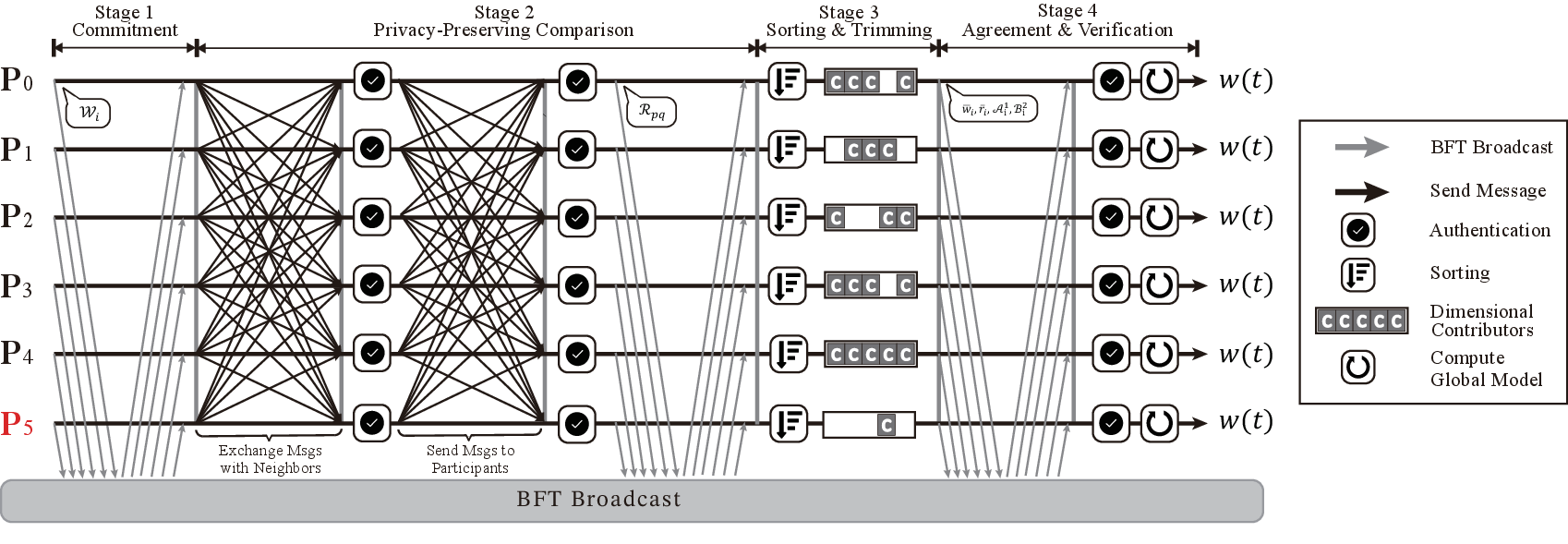}
  \caption{This figure depicts a P2P FL with six participants $\{P_0,\ldots,P_5\}$ and illustrates 
   the schematic message flow of \texttt{Brave}. In this example, $P_5$ is a Byzantine adversary. In Stage 1, all participants broadcast their commitments of local models to BFT broadcast module. In Stage 2, the participants first exchange messages among themselves using the point-to-point communication network. After that, they send a set of relationships among local models to BFT broadcast. In Stage 3, the participant sorts the local models along each coordinate after receiving the pairwise comparison results. It then trims the largest and smallest $f$ values in each coordinate. In Stage 4, MPC is is performed to aggregate models while preserving privacy.}
  \label{fig: communication structure} 
\end{figure*}

\subsection{Stage 1: Commitment}

In what follows, we introduce the Commitment stage of \texttt{Brave}. In this stage, each $P_i$ first updates its local model ${w}_i$ using Eqn. \eqref{equation weight update}. 
Then, a one-to-one strictly monotonic mapping is performed to map floating point representation of ${w}_i$ to $\mathbb{Z}_q$ with $q$ being a large number, and therefore, ${w}_i \in \mathbb{Z}_q^m$. 
If not specified, we assume all ${w}_i \in \mathbb{Z}_q^m$ when the context is clear. The protocol then generates a private random helper ${r^w_{i}}\in\mathbb{Z}^m_q$.
Given local model ${w}_i$ and random helper ${r^w_{i}}$, $P_i$ next generates a \textit{commitment} of ${w}_i$ as $\mathcal{W}_{i}=\mathcal{C}({w_i},{r^w_{i}})$, where $\mathcal{C}$ represents the commitment operation (we defer the detailed introduction on commitment later). 
We say that ${w}_i$ w.r.t. $\mathcal{W}_{i}$ is the \textit{claimed local model} of $\mathcal{W}_{i}$. Note that for a benign participant, ${w}_i$ is exactly the local model, but for a Byzantine participant, the claim of $\mathcal{W}_{i}$ may not be equal to its local model. The commitment $\mathcal{W}_{i}$ is then broadcast to all other participants using BFT broadcast. The details of the BFT broadcast can be found in Appendix \ref{Appendix: Design Specification}.

\texttt{Brave} utilizes the Pedersen commitment scheme for $\mathcal{C}$, defined as follows.
\begin{definition}[Pedersen Commitment Scheme \cite{feigenbaum_non-interactive_1992}]
Let $G_q$ denote a group of prime order $q$, such that the discrete logarithm problem in this group is infeasible. Let $g$ and $h$ denote independent generators of $G_q$ that are known to all participants.
The commitment $\mathcal{W}_i$ of local model ${w}_i \in \mathbb{Z}_q$ is generated as
\begin{equation}\label{equation commitment}
	\mathcal{C}({w}_i,{r}_i^{{w}}):=g^{{w}_i}h^{{r}_i^{{w}}},
\end{equation}
where exponents and multiplication are coordinate-wise.
\label{def Pedersen commitment}
\end{definition}

The Pedersen commitment in Eqn. \eqref{equation commitment} provides our protocol with the following properties. 
First, the commitment $\mathcal{W}_i$ does not reveal any information about the local model ${w}_i$, and thus ${w}_i$ is \emph{information-theoretically hiding} \cite{feigenbaum_non-interactive_1992}. 
Second, the commitment is \emph{computationally binding} in the sense that one cannot solve for a local model ${w}_i'$ such that $\mathcal{C}({w}_i,{r}_i^{{w}})=\mathcal{C}({w}_i',{r}_i^{{w}'})$ to open the commitment if it is unable to solve discrete logarithm problems.
Finally, the commitments obtained using Eqn. \eqref{equation commitment} are \emph{additively homomorphic}, i.e., 
\begin{equation}
	\mathcal{C}({w}_1,{r}_1^{{w}}) \mathcal{C}({w}_2,{r}_2^{{w}}) 
	= \mathcal{C}({w}_1+{w}_2,{r}_1^{{w}}+{r}_2^{{w}}).
\label{equation additively homomorphic}
\end{equation}
These properties are summarized in the following lemma \ref{lemma Pedersen commitment}.

\begin{lemma}[\cite{feigenbaum_non-interactive_1992}]
The commitment procedure in Definition \ref{def Pedersen commitment} is computationally binding and additively homomorphic. Furthermore, it ensures that the local model ${w}_i$ is information-theoretically hiding.
\label{lemma Pedersen commitment}
\end{lemma}

In our design, the commitment scheme ``locks'' the initial input and prevents an adversary node from modifying it in later steps. We prove that the proposed commitment allows us to achieve the information-theoretic privacy and Byzantine resilience of \texttt{Brave} in the Appendix. For a comprehensive understanding, detailed pseudocode of this stage is provided in Algorithm \ref{algorithm commitment} in Appendix \ref{Appendix: Design Specification}.

\subsection{Stage 2: Privacy-Preserving Comparison}
In this stage, each $P_i$ conducts a privacy-preserving comparison, which coordinate-wise compares its model with all other participants. To preserve the privacy of the true values in the local model, the comparison is achieved by sharing a \textit{masked} value obtained by adding a random number to the true value, then canceling it during pairwise comparison.

Specifically, $P_i$ generates two random vectors ${v_{ij}}, {r^{v}_{ij}} \in \mathbb{Z}_q^m$ associated with each $P_j$. 
It then calculates a masked local model of $P_i$ with respect to $P_j$ as ${c_{ij}}={w_i}+{v_{ij}}$, and sums two random helpers as ${r^{w+v}_{ij}}={r^w_{i}}+{r^v_{ij}}$ for the purpose of authentication. After that, it computes the commitment of ${v_{ij}}$, denoted as $\mathcal{V}_{ij}$, using Eqn. \eqref{equation commitment} and random helper $r_{ij}^v$, and sends ${c_{ij}}$, ${r^{w+v}_{ij}}$, as well as $\mathcal{V}_{ij}$ to $P_j$. 
In this communication, there are also some random vectors ${s^1_{ij}}$, ${r^{s1}_{ij}}$, ${s^2_{ij}}$, and ${r^{s2}_{ij}}$ being exchanged, which serve as padding randomness in the next stage.
After receiving the message from $P_j$, $P_i$ leverages the additively homomorphic property of commitment (Lemma \ref{lemma Pedersen commitment}) to verify the authenticity of $P_j$'s masked local model $c_{ji}$ using 
\begin{equation}
\setlength{\abovedisplayskip}{3pt}
\setlength{\belowdisplayskip}{3pt}
    \mathcal{W}_j \cdot \mathcal{V}_{ji}=\mathcal{C}(c_{ji},r^{w+v}_{ji}), 
    \label{equation verification}
\end{equation}
where $\mathcal{W}_j$ is the commitment generated by $P_j$ in Stage 1.
If Eqn. \eqref{equation verification} holds for every coordinate of ${c_{ji}}$, $P_i$ calculates ${d_{ij}}={v_{ij}}+{c_{ji}}$ and $r^d_{ij}=r^v_{ij}+r^{w+v}_{ji}$. It then sends a message $m_{ij}:=[{d_{ij}}$, ${r^d_{ij}}$,$\mathcal{V}_{ij}$,$\mathcal{V}_{ji}$] to all other participants except $P_i$. As shown in Fig. \ref{fig: communication structure}, we do not use BFT broadcast here, since $P_i$ can never know ${d_{ij}}$ — if it knows, it can calculate ${v_{ji}}$, and then get ${P_j}$'s local model ${w_j}={{c_{ji}-v_{ji}}}$. 

After exchanging messages with all other participants, each $P_i$ now coordinate-wise compares the local models of any other two participants denoted as $w_q$ and $w_p$ with $q,p \neq i$. 
$P_i$ first check if $\mathcal{V}_{pq}$ or $\mathcal{V}_{qp}$ from $P_q$ and $P_p$ are the same, If they are not the same, then at least one of $P_q$ and $P_p$ is malicious. 
The participant then verifies whether the following two relations 
\begin{equation}
\setlength{\abovedisplayskip}{5pt}
\setlength{\belowdisplayskip}{5pt}
	\left\{\begin{array}{l}
{d_{pq}}={v_{pq}+c_{qp}}={v_{pq}}+{v_{qp}}+{w_q} \\
{d_{qp}}={v_{qp}}+{c_{pq}}={v_{qp}}+{v_{pq}}+{w_p}
\end{array}\right.
\label{equation comparison}
\end{equation}
hold true or not by calculating the commitment of $d_{pq}$ and $d_{qp}$ via $\mathcal{V}_{pq}$, $\mathcal{V}_{qp}$, $\mathcal{W}_{p}$, and $\mathcal{W}_{q}$.
Note that ${v_{pq}}+{v_{qp}}$ is a common term in Eqn. \eqref{equation comparison}. Therefore, if Eqn. \eqref{equation comparison} holds for all coordinates, then $P_i$ can obtain a coordinate-wise relationship $\bowtie$ between $w_p^k$ and $w_q^k$ for each coordinate $k$, where $\bowtie \in \{<,>\}$. We assume ties are broken arbitrarily, so that an equality relationship does not occur.
Finally, $P_i$ broadcasts the relationship of each $P_p$ and $P_q$ along each coordinate $k$ to all participants using BFT broadcast. Please refer to Algorithm \ref{Participant Only: algorithm aggregator comparison and sorting} in Appendix \ref{Appendix: Design Specification} for detailed information.

\subsection{Stage 3: Sorting \& Trimming}

The Sorting \& Trimming stage starts after the relationships are broadcast. In this stage, $P_i$ first counts the relationships it received from BFT broadcast. If a relationship $w_p^k \bowtie w_q^k$ appears more than $2f$ times, it is accepted by $P_i$. Then in each coordinate $k$, $P_i$  sorts $\{w_j^k: P_j \in \mathbb{P} \setminus {P_i}\}$ coordinate-wise in ascending order based on the accepted relationships as $\mathcal{S}^k_i:=w_a^k < \ldots < w_b^k$.
We denote the length of $\mathcal{S}^k_i$ as $N_i^k \leq N$, where the equality holds when all local models of the participants can be sorted in $\mathcal{S}^k_i$ using $\mathcal{R}_{pq}$.
After sorting, $P_i$ trims the lowest $f$ and highest $f$ values in each $\mathcal{S}^k_i$. We denote the remaining participants in the trimmed $\mathcal{S}^k_i$ as \emph{contributors} of coordinate $k$, denoted as $\mathbb{C}^k_i \subset \mathbb{P}$. Detailed protocol can be found in Algorithm \ref{stage 3: contributor selection} in Appendix \ref{Appendix: Design Specification}.

\subsection{Stage 4: Aggregation \& Verification}

Given the contributors from Stage 3, the last stage, Aggregation and Verification, sums the local models of contributors and updates a global model. Specifically, each $P_i$ calculates a cloaked local model $\bar{w}_i$ within each coordinate $k$ as follows
\begin{equation}
\setlength{\abovedisplayskip}{4pt}
\setlength{\belowdisplayskip}{4pt}
    \bar{w}_i^k:=\left(w_i^k-\sum_{j \in \mathbb{C}^k_{-i}} a^k_{i j}+\sum_{j \in \mathbb{C}^k_{+i}} a^k_{i j}\right) \bmod \beta,
    \label{equation one pad}
\end{equation}
where $\mathbb{C}^k_{-i} \subset \mathbb{C}_i^k$ (resp. $\mathbb{C}^k_{+i} \subset \mathbb{C}_i^k$) is the set of contributors whose local models are lower (resp. greater) than $w_i^k$ in $k$-th coordinate, $\beta$ is a commonly known large prime number, and $a_{i j}$ is an agreed randomness between $P_i$ and $P_j$. A similar calculation will be performed to obtain the summation of the cloaked random helpers for verification.
Then $P_i$ sends the cloaked $\bar{w}_i$ and $\bar{r}_i$ with corresponding commitments $\mathcal{A}_i$ and $\mathcal{B}_i$ to all participants via BFT broadcast.

After receiving $\bar{w}_i$ and $\bar{r}_i$ from the BFT broadcast, all participants then sum the cloaked values and calculate the global model $\bar{w}$. Each participant then verifies if
\begin{equation}
\setlength{\abovedisplayskip}{4pt}
\setlength{\belowdisplayskip}{4pt}
	\mathcal{C}(\bar{w}^k,\bar{r}^k) = \prod_{P_j \in \mathbb{C}_i^k} \mathcal{W}_j^k
 \label{equation final verification}
\end{equation} 
holds in each coordinate. If Eqn. \eqref{equation final verification} holds, $P_i$ computes $w^k$ as $\frac{\bar{w}^k}{N_i^k-2f}$ for each coordinate $k$.
$P_i$ next projects $w^k$ to the floating numbers, and completes one iteration of P2P FL. The pseudocode is provided in Algorithm \ref{Participant Only: algorithm sum} in Appendix \ref{Appendix: Design Specification}.

We note that if Eqn. \eqref{equation final verification} does not hold, then $P_i$ can identify the Byzantine adversaries thanks to the commitment scheme. Specifically, for each contributor $P_j\in\mathbb{C}_i^k$, if 
\begin{equation}
    \mathcal{C}(\bar{w}_j^k,\bar{r}^k) \neq \mathcal{W}_j^k \prod_{P_h \in \mathbb{C}^k_{-i}} \mathcal{A}_{ih}^k \prod_{P_h \in \mathbb{C}^k_{+i}} (\mathcal{A}_{ih}^k)^{-1},
\end{equation}
then $P_i$ flags $P_j$ as a Byzantine adversary. We term this procedure as \textit{blame}, and the detailed algorithm can be found in Algorithm \ref{Participant Only: algorithm blame} in Appendix \ref{Appendix: Design Specification}.

\subsection{Privacy and Resilience Guarantees of \texttt{Brave}}

We present the information-theoretic privacy and Byzantine resilience of \texttt{Brave} in Appendix \ref{sec: Byzantine Resilience Analysis}.
Specifically, we leverage Lemma \ref{lemma Pedersen commitment} of the Pedersen Commitment scheme to prove information-theoretic privacy. 
To show the Byzantine resilience of \texttt{Brave}, we prove $\epsilon$-\textit{convergence} and agreement, respectively. We show that the global model obtained by the benign participants deviates from the optimal one by a bounded distance, i.e., satisfies $\epsilon$-convergence property.
We then prove agreement by showing that the relationship $\bowtie$ accepted by benign participants preserves the correct ordering of the claimed local models, then prove the global model of all benign participants is the same using BFT broadcast.

\section{Experiments}
\label{sec: performance evaluation}
We evaluate \texttt{Brave} against three state-of-the-art adversaries using two image classification tasks. We demonstrate that \texttt{Brave} ensures the models learned by P2P FL achieve comparable accuracy to global models trained in the absence of any adversary, and hence resilient to the adversaries.

\begin{table*}[htbp]
\centering
\caption{This table presents the classification accuracy of the learned 2NN and CNN using a P2P FL with $N=10$ and $f=2$. The second row of the table represents the scenario where the Byzantine adversaries send their true local models and do not launch any attack.
Rows 3-6 correspond to different threat models. We observe that \texttt{Brave} ensures the P2P FL to learn a global model with near-optimal classification accuracy against all threat models, and hence is Byzantine-resilient.}
\begin{tabular}{c c c c c} 
\toprule
      \multirow{2}{*}{Adversary Strategy}& \multicolumn{2}{c}{w/o \texttt{Brave}} & \multicolumn{2}{c}{\texttt{Brave}}\\
\cmidrule{2-5}
    & \texttt{2NN+MNIST} & \texttt{CNN+CIFAR10} & \texttt{2NN+MNIST} & \texttt{CNN+CIFAR10}\\
\midrule
    No Attack & $\mathbf{97.35}\%$ & $\mathbf{63.94}\%$ & $97.21\%$ & $63.55\%$\\
    \makecell{Label Flip} & $89.91\%$ & $52.15\%$ & $\textbf{96.74}\%$ & $\mathbf{60.91}\%$\\
    \makecell{Sign Flip} & $11.35\%$ & $48.68\%$ & $\textbf{97.02}\%$ & $\mathbf{63.54}\%$\\
    \makecell{Gaussian $(\sigma = 0.1)$} & $92.02\%$ & $55.58\%$ & $\textbf{96.92}\%$ & $\mathbf{63.08}\%$\\
    \makecell{Gaussian $(\sigma = 1)$} & $53.01\%$ & $10.01\%$ & $\textbf{97.12}\%$ & $\mathbf{61.92}\%$\\
\hline
\end{tabular}
\label{table:Acc ALL}
\end{table*}

\begin{table}[h!]
\centering
\caption{This table presents the accuracy of the 2NN model trained by P2P FL with \texttt{Brave} involving $N=10,15,20$ participants. The number of Byzantine adversaries is set as $f/N=20\%$. We observe that \texttt{Brave} ensures near-optimal classification accuracy with different choices of $N$, and \texttt{Brave} is insensitive to parameter $N$.}

\begin{tabular}{cc cc} 
\toprule
    Adv. Strategy & $N=10$ & $N=15$ & $N=20$\\
\midrule
    No Attack & $97.21\%$ & $97.54\%$ & $ \mathbf{97.64}\%$\\
    Label Flip & $96.74\%$ & $97.18\%$ & $ \mathbf{97.50}\%$\\
    Sign Flip & $97.02\%$ & $97.34\%$ & $\mathbf{97.51}\%$\\
    Gaussian ($\sigma=0.1$) & $96.92\%$ & $97.39\%$ & $\mathbf{97.42}\%$\\
    Gaussian ($\sigma=1$) & $97.12\%$ & $97.27\%$ & $\mathbf{97.59}\%$\\
\bottomrule
\end{tabular}
\label{table:impact of N}
\end{table}

\subsection{Experimental Setup}

\underline{\emph{Datasets}}:
We use two benchmark datasets: CIFAR10 \cite{krizhevsky2009learning} and MNIST \cite{lecun2010mnist}, for image classification tasks.
CIFAR10 consists of $50000$ training and $10000$  testing images, each of size $32\times 32$.
Each image within CIFAR10 belongs to one of ten classes.
There are $60000$ training and $10000$ testing images in MNIST dataset. 
Each image is of size $28\times 28$, and can be classified to one out of ten classes.

\underline{\emph{\texttt{Brave} Setup}}:
We implement \texttt{Brave} on two P2P FL with different settings.
In the first P2P FL, we let the participants train a 2-hidden-layer model (2NN) using samples from the MNIST dataset.
In the second P2P FL, the participants learn a Convolutional Neural Network (CNN) model using CIFAR10 dataset.
In both P2P FL, the training images from CIFAR10 or MNIST dataset are independently and identically  distributed (i.i.d.) to the participants so that each participant has $|\mathcal{D}_i|=2000$ images within its private dataset.
The participants update their local models $w_i(t)$ using stochastic gradient descent (SGD) algorithm with learning rate $\eta=0.01$ \cite{mcmahan2017communication}.


\underline{\emph{Baseline Setup}}: We present the effectiveness of \texttt{Brave} by comparing with a baseline, P2P FL-na\"ive. The baseline implements the classic P2P FL as given in Eqn. \eqref{equation weight update} and \eqref{equation directly average weight}.


\underline{\emph{Threat Models}}: 
We evaluate \texttt{Brave} against Byzantine adversaries who adopt distinct strategies, detailed as below.
\begin{itemize}
    \item \emph{No Attack}: The adversaries does not initiate any attack and behave as benign participants.
    \item \emph{Label Flip Attack \cite{tolpegin2020data-labelflipping}}: Label flip attack is a data poisoning attack. A Byzantine adversary launches a label flip attack by training its local model using mislabeled data. As a consequence, a global model that is corrupted by label flip attacks may misclassify the input image, and hence exhibits low classification accuracy.
    \item \emph{Sign Flip Attack \cite{xie2019zeno-bitflipping}}: A Byzantine adversary carries out a sign flip attack by first flipping the sign of its local model and then sending it to the other participants. As a consequence, the original local model $w_i$ of a Byzantine adversary is manipulated as $-w_i$. 
    \item \emph{Gaussian Attack \cite{fang2020local-gaussian}}: Gaussian attack is essentially an instantiation of a model poisoning attack. A Byzantine adversary manipulates its local model by adding a zero-mean Gaussian noise $y^k\sim\mathcal{N}(0,\sigma^2)$ to $w_i^k(t)$ along each dimension $k$. 
    Here $\sigma$ is the standard deviation that can be tuned by the Byzantine adversary.
    In our experiments, we evaluate the scenarios where $\sigma$ are chosen as $0.1$ and $1$.
\end{itemize}

The state-of-the-art white-box membership inference attack (MIA) \cite{zhang2020gan} requires the passive adversaries to access the local models of other participants and the architecture of the neural network.
By our design of \texttt{Brave}, the passive adversaries can only receive the masked local models and commitments, whereas the local models are hidden.  
Therefore, the white-box MIA is infeasible for P2P FL when \texttt{Brave} is implemented.

\underline{\emph{Evaluation Metric}}: We use classification accuracy over the testing dataset as the evaluation metric. It captures the fraction of input images sampled from the testing dataset that can be correctly classified by using the learned model.


\subsection{Experimental Results}
In what follows, we demonstrate the effectiveness of \texttt{Brave}.

\underline{\emph{Byzantine Resilience of \texttt{Brave}}}: 
We evaluate the classification accuracy of the learned global models obtained by P2P FL with $N=10$ participants and $f=2$ Byzantine adversaries.
In Table \ref{table:Acc ALL}, we present the classification accuracy of the learned 2NN and CNN when the Byzantine adversaries adopt different attack strategies.
We observe that if the Byzantine adversaries does not initiate any attack, then \texttt{Brave} retains comparable accuracy compared with classic P2P FL that does not implement \texttt{Brave}.
Furthermore, \texttt{Brave} guarantees significantly higher accuracy of the learned model once the Byzantine adversaries send compromised local models to the other participants, and thereby is Byzantine-resilient.

\begin{figure*}[ht]
\centering
                 \begin{subfigure}{.32\textwidth}
                 \includegraphics[width=\textwidth]{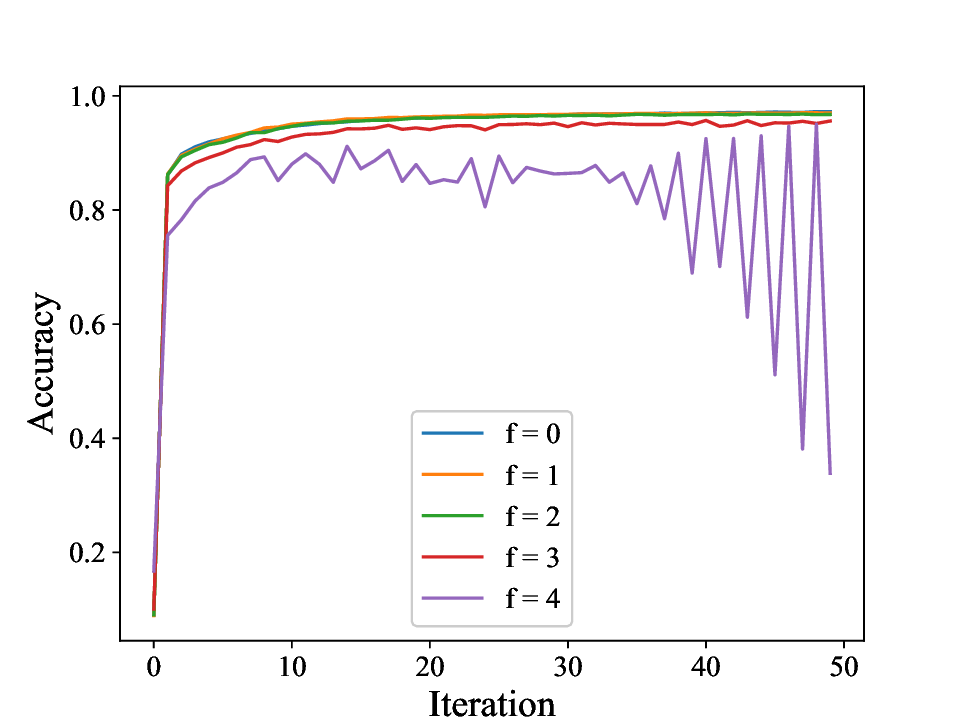}
                 \subcaption {Label Flip Attack}
                 \label{fig:converge_2nn_label}
                 \end{subfigure}\hfill
                 \begin{subfigure}{.32\textwidth}
                 \includegraphics[width=\textwidth]{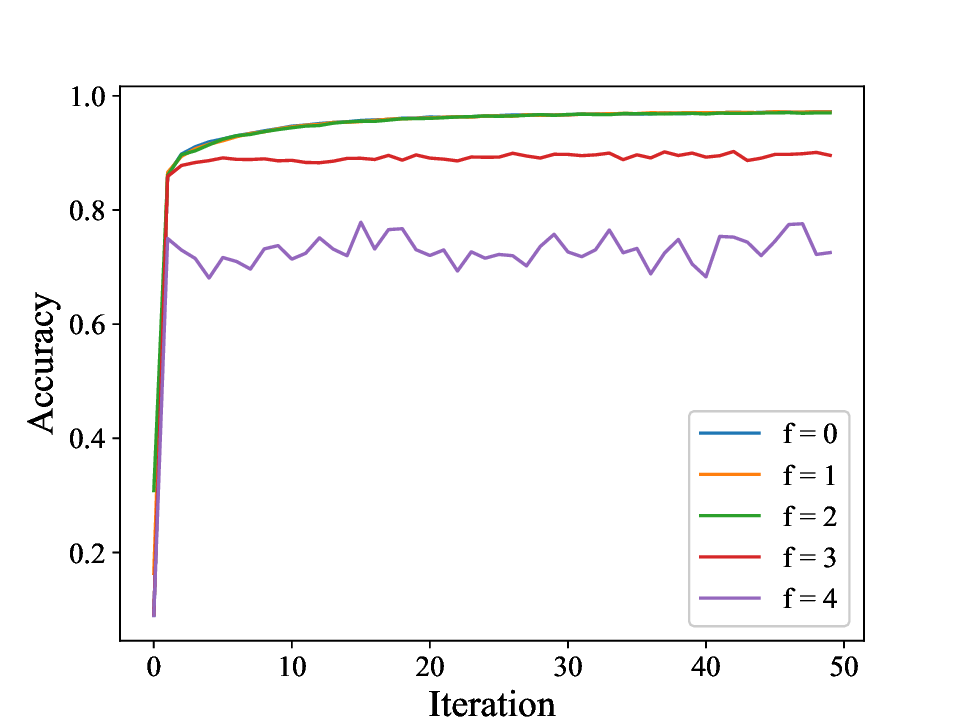}
                 \subcaption {Sign Flip Attack}
                 \label{fig:converge_2nn_sign}
                 \end{subfigure}\hfill
                 \begin{subfigure}{.32\textwidth}
                 \includegraphics[width=\textwidth]{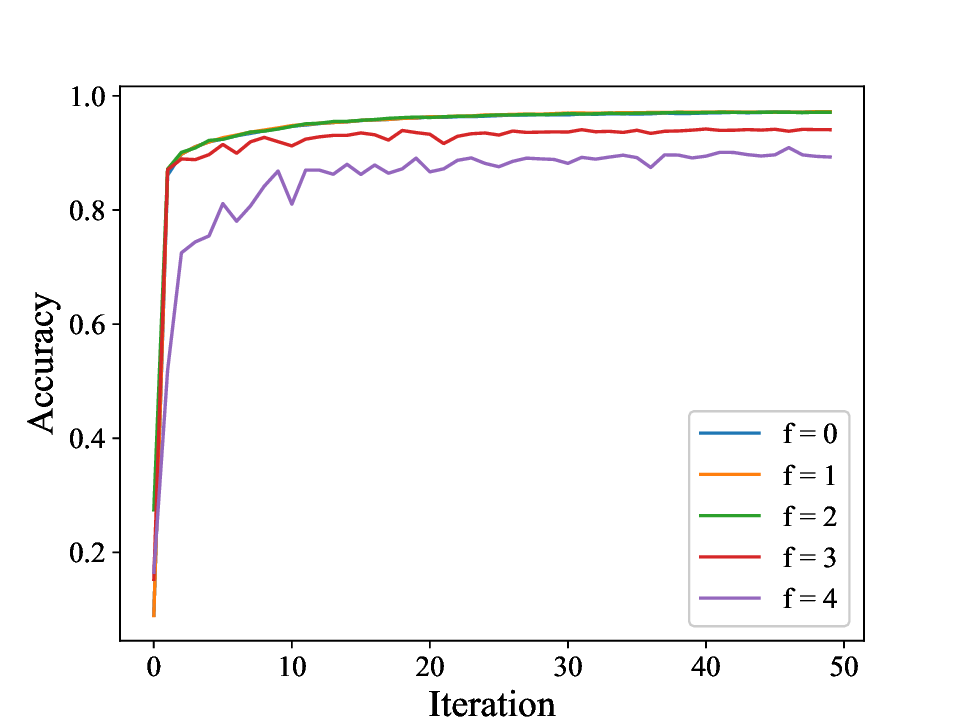}
                 \subcaption{Gaussian Attack ($\sigma=1$)}
                 \label{fig:converge_2nn_gaussian}
                 \end{subfigure}\hfill
\caption{This figure presents the accuracy of 2NN learned using P2P FL with $N=10$ participants at each iteration $t$. When the number of Byzantine adversaries $f$ satisfies $N>3f+2$, i.e., $f\in\{0,1,2\}$,  \texttt{Brave} ensures $\epsilon$-convergence property as given in Definition \ref{def:byzantine resilience}. When $f$ violates $N>3f+2$, the Byzantine adversaries can corrupt the learned 2NN, and even prevents FL from converging (Fig. \ref{fig:converge_2nn_label}, $f=4$).}
\end{figure*}

\begin{figure*}[ht]
\centering
                 \begin{subfigure}{.31\textwidth}
                 \includegraphics[width=\textwidth]{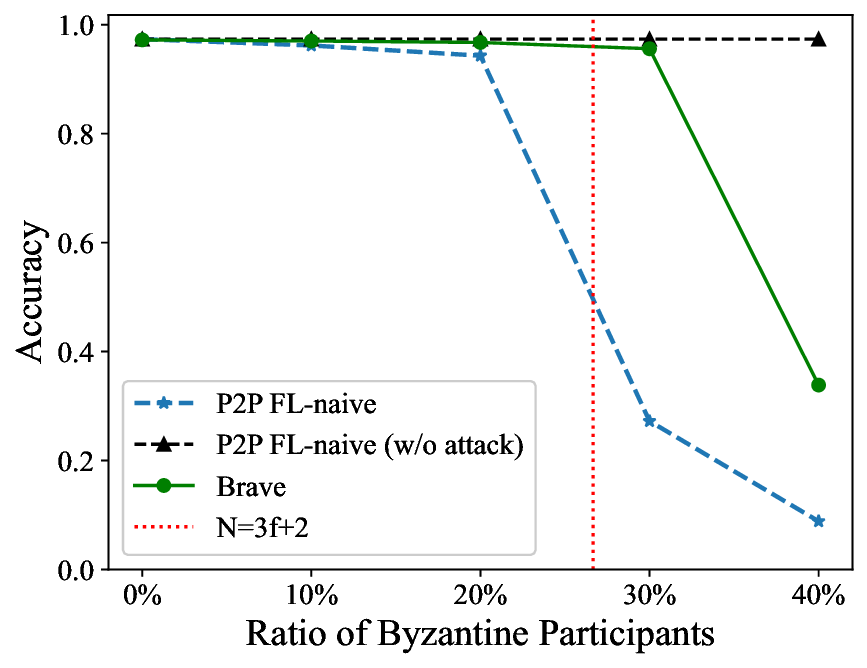}
                 \subcaption {Label Flip Attack}
                 \label{fig:ratio_2nn_label}
                 \end{subfigure}\hfill
                 \begin{subfigure}{.31\textwidth}
                 \includegraphics[width=\textwidth]{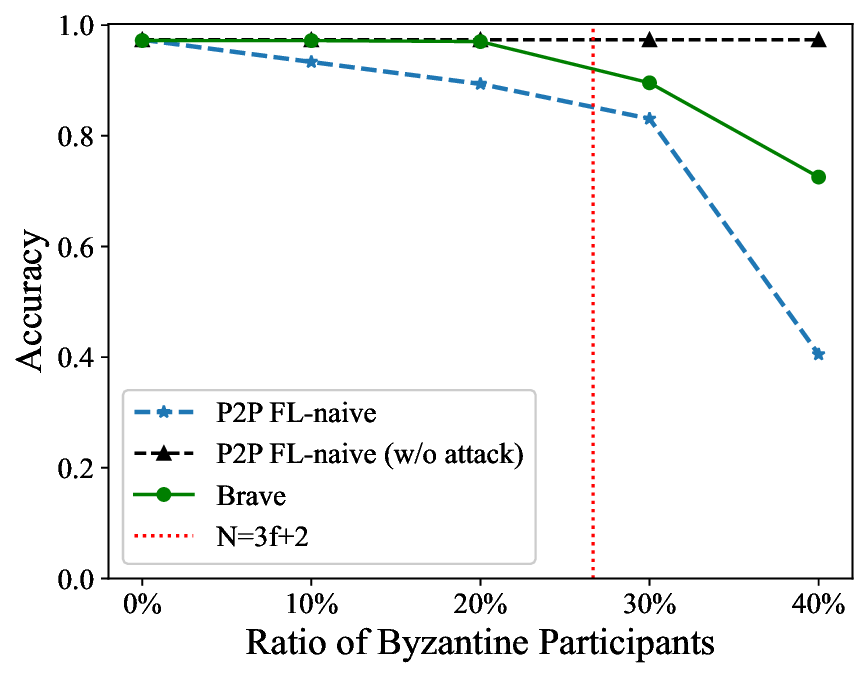}
                 \subcaption {\small{Sign Flip Attack}}
                 \label{fig:ratio_2nn_sign}
                 \end{subfigure}\hfill
                 \begin{subfigure}{.31\textwidth}
                 \includegraphics[width=\textwidth]{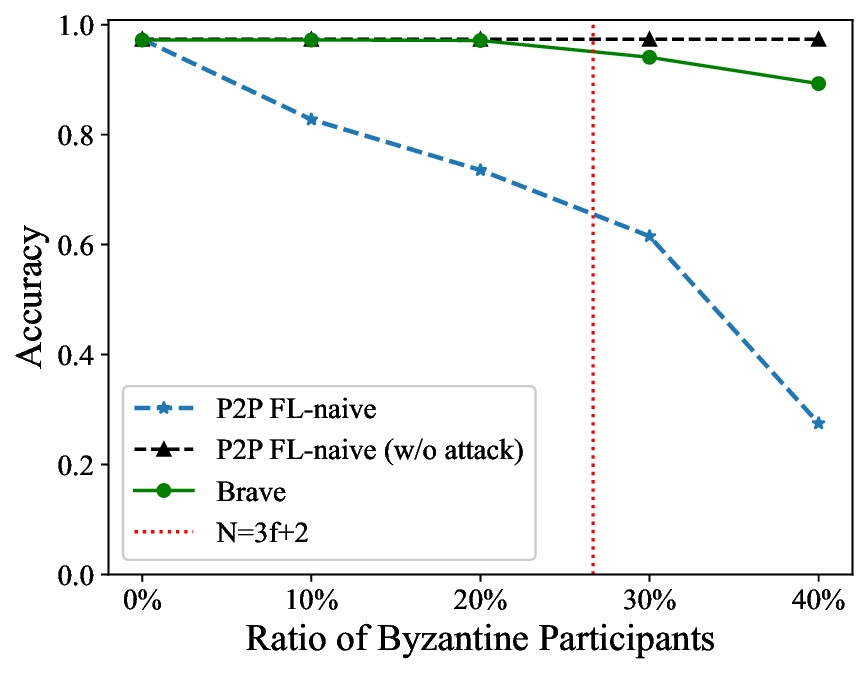}
                 \subcaption{Gaussian Attack ($\sigma=1$)}
                 \label{fig:ratio_2nn_gaussian}
                 \end{subfigure}\hfill
\caption{This figure presents how accuracy of the learned 2NN varies with respect to the ratio of Byzantine participants when $N=10$. If $N>3f+2$ holds (the left-hand side of the vertical red line), then \texttt{Brave} guarantees Byzantine resilience as the performance of learned 2NN is comparable to P2P FL-na\"ive learned without any adversary.}
\end{figure*}

\underline{\emph{$\epsilon$-convergence of \texttt{Brave}}}:  
We demonstrate $\epsilon$-convergence of \texttt{Brave} in Fig. \ref{fig:converge_2nn_label}-\ref{fig:converge_2nn_gaussian}.
We observe that if $N>3f+2$ holds, then \texttt{Brave} guarantees that the global model learned in the presence of Byzantine adversaries remains $\epsilon$-close to the global model learned when $f=0$. 
We further notice that the value of $\epsilon$ increases with respect to the number of Byzantine adversaries.
The presence of more Byzantine adversaries could introduce additional bias to the learned global model, or even cause FL to fail (as shown in Fig. \ref{fig:converge_2nn_label} when $f=4$).

\underline{\emph{Insensitivity to the Number of Participants $N$}}:  
We evaluate the effectiveness of \texttt{Brave} when applying to P2P FL of different scales, i.e., the number of participants is $N=\{10,15,20\}$.
In all three settings, we fix the ratio of Byzantine participants as $f/N=20\%$.
We find that the classification accuracy increases with respect to the number of participants $N$, given that $f/N$ remains fixed.
Such increment of accuracy is because more training data is used as $N$ increases. 
Therefore, we believe that \texttt{Brave} is insensitive to the scale of P2P FL. 
For P2P FL of different scales, \texttt{Brave} guarantees their Byzantine resilience when $N>3f+2$ holds.

\underline{\emph{Sensitivity to the Ratio of Byzantine Adversaries $f/N$}}:  
We evaluate the effectiveness of \texttt{Brave} when the ratio of Byzantine adversaries $f/N$ varies in Fig. \ref{fig:ratio_2nn_label}-\ref{fig:ratio_2nn_gaussian}. We consider the 2NN learned using P2P FL with $N=10$. We observe that \texttt{Brave} can maintain high classification accuracy (stay close to the orange curve with triangle markers) when $f$ satisfies $N>3f+2$ (the left-hand side of the vertical red line).
When P2P FL does not implement \texttt{Brave}, a significant decrease in accuracy can be observed from the blue curve even when $N>3f+2$ holds.
Once $N>3f+2$ is violated, the Byzantine resilience guarantee provided by \texttt{Brave} decreases. However, P2P FL with \texttt{Brave} implemented still outperforms P2P FL-na\"ive which does not adopt \texttt{Brave}.

\section{Conclusion and Future Work}
\label{sec: conclusion}

In this paper, we considered peer-to-peer federated learning in the presence of both passive and Byzantine adversaries.
The passive adversaries aimed at inferring the other participants' private information during the training process, whereas Byzantine adversaries could arbitrarily manipulate the information it sent to disrupt the learning algorithm.
We developed a four-stage P2P FL protocol named \texttt{Brave} that information-theoretically preserves privacy and is resilient to malicious attacks caused by Byzantine adversaries.
We evaluated \texttt{Brave} using two image classification tasks with CIFAR10 and MNIST datasets. 
Our results showed that \texttt{Brave} can effectively defend against the state-of-the-art adversaries.
In future work, we will reduce the complexity of \texttt{Brave} and incorporate the delays incurred on the communication links when participants exchange information.

\section*{Acknowledgements}
This work is partially supported by the National Science Foundation (NSF) under grant No. 2229876 and Air Force Office of Scientific Research (AFOSR) under grant FA9550-23-1-0208.

This work is supported in part by funds provided by the National Science Foundation, by the Department of Homeland Security, and by IBM. 
Any opinions, findings, and conclusions or recommendations expressed in this material are those of the author(s) and do not necessarily reflect the views of the National Science Foundation or its federal agency and industry partners.

\bibliography{references}

\appendix
\newpage
\appendix
\onecolumn

\clearpage

\section{Detailed Algorithms of \texttt{Brave}}
\label{Appendix: Design Specification}

\subsection{BFT Broadcast} All stages of \texttt{Brave} require the participants to broadcast messages via a module named BFT broadcast.
BFT broadcast utilizes Byzantine-fault-tolerant (BFT) distributed consensus algorithms, which follow a propose-vote paradigm \cite{shi2020foundations}.
In the BFT broadcast, the message sent by each participant is treated as a proposal.
All participants will incorporate the proposal to their local information set as long as more than $N-f$ participants agree with the proposal.
Typical implementations of the BFT distributed consensus algorithms include PBFT \cite{castro1999practical}, Zyzzyva \cite{kotla2007zyzzyva}, and Hotstuff \cite{yin2019hotstuff}.
Using the state-of-the-art implementation of BFT broadcast, Hotstuff \cite{yin2019hotstuff}, \texttt{Brave} guarantees that all participants will reach an agreement on messages sent by the participants, with communication complexity linear in $N$.

\subsection{Stage 1: Commitment}

After updating its local model ${w}_i(t)$ in each iteration, participant $P_i$ generates a \textit{commitment} of ${w}_i(t)$. Commitment allows participants to pledge to a chosen value while keeping it hidden from others, with the ability to reveal that value later.

\begin{algorithm}
\caption{Stage 1: Commitment}
\label{algorithm commitment} 
\begin{algorithmic}[1]
\FORALL{$P_i \in \mathbb{P}$ in parallel}
    \STATE ${w}_i =$ Update(${w}(t), \mathcal{D}_i$) \label{algorithm commitment line 1}
    \STATE ${r^w_{i}} \gets$ Generate a random helper for commitment
    \STATE $\mathcal{W}_{i}=\mathcal{C}({w_i},{r^w_{i}})$ \label{algorithm commitment line 4}
    \STATE \texttt{BFT-Broadcast}($\mathcal{W}_i$)
\ENDFOR
\end{algorithmic}
\end{algorithm}

\subsection{Stage 2: Privacy-preserving Comparison}

In this stage, participant $P_i$ conducts a privacy-preserving comparison, which coordinate-wise compares its model with all other participants.

\begin{algorithm}[h!]
\caption{Stage 2: Privacy-Preserving Comparison}
\label{Participant Only: algorithm aggregator comparison and sorting} 
\begin{algorithmic}[1]
\FORALL{$P_i \in \mathbb{P}$ in parallel}
    \FORALL{$P_j \in \mathbb{P} \setminus \{P_i\}$ in parallel} \label{algorithm pairwise comparison starts}
        \STATE Generate random vectors ${v_{ij}}, {r^v_{ij}}$, ${s^1_{ij}}, {r^{s1}_{ij}}, {s^2_{ij}}, {r^{s2}_{ij}}$ in $\mathbb{Z}_q^m$\; \label{algorithm generate random vectors}
        \STATE ${{c_{ij}}}={v_{ij}}+{w_i}$; ${r^{w+v}_{ij}}={r^w_{i}}+{r^v_{ij}}$
        \STATE $\mathcal{V}_{ij}=\mathcal{C}({v_{ij}},{r^v_{ij}})$
        \STATE Send [$\mathcal{V}_{ij}$, ${c_{ij}}$, ${r^{w+v}_{ij}}$, ${s^1_{ij}}$, ${r^{s1}_{ij}}$, ${s^2_{ij}}$, ${r^{s2}_{ij}}$] to $P_j$
        \STATE Receive [$\mathcal{V}_{ji}$, ${c_{ji}}$, ${r^{w+v}_{ji}}$, ${s^1_{ji}}$, ${r^{s1}_{ji}}$, ${s^2_{ji}}$, ${r^{s2}_{ji}}$] from $P_j$ \label{algorithm line receive message from neighbor}
        \IF{$\mathcal{W}_j \cdot \mathcal{V}_{ji}=\mathcal{C}(c_{ji},r^{w+v}_{ji})$} \label{line participant check}
            \STATE $d_{ij}=v_{ij}+c_{ji}$; $r^d_{ij}=r^v_{ij}+r^{w+v}_{ji}$
            \STATE Send $m_{ij}$=[$d_{ij}$, $r^d_{ij}$, $\mathcal{V}_{ij}$, $\mathcal{V}_{ji}$] to $\mathbb{P} \setminus \{P_j\}$ \label{algorithm line send messages to participants}
        \ENDIF
    \ENDFOR
\WHILE {receive $m_{pq}$ and $m_{qp}$ from $P_p$ and $P_q$}
    \IF {same values of $\mathcal{V}_{pq}$ or $\mathcal{V}_{qp}$ from $m_{pq}$ and $m_{qp}$} \label{algorithm line authenticate msgs}
    \IF {$\mathcal{C}(d_{pq},r^d_{pq}) = \mathcal{V}_{pq} \cdot \mathcal{V}_{qp} \cdot \mathcal{W}_q$ and $\mathcal{C}(d_{qp},r^d_{qp}) = \mathcal{V}_{qp} \cdot \mathcal{V}_{pq} \cdot \mathcal{W}_p$}
        \FOR{\textbf{all} $k \in [m]$} \label{algorithm line for all} 
            \STATE $\mathcal{R}_{pq}=\{w_{p}^k \bowtie w_{q}^k : \forall k\}$
        \ENDFOR
        \STATE \texttt{BFT-Broadcast}($\mathcal{R}_{pq}$) \label{algorithm line BB for relation}
    \ENDIF
    \ENDIF
\ENDWHILE
\ENDFOR
\end{algorithmic}
\end{algorithm}

\subsection{Stage 3: Sorting and Trimming}
In this stage, the participant sorts the local models along each coordinate after receiving the pairwise comparison results. It then trims the largest and smallest $f$ values in each coordinate and denotes the remainder as \textit{contributors}.

\begin{algorithm}[h!]
\caption{Stage 3: Contributor Selection}
\label{stage 3: contributor selection}
\begin{algorithmic}[1]
\FORALL{$P_i \in \mathbb{P}$ in parallel}
\FOR{\textbf{all} $\mathcal{R}_{pq}$ from BFT Broadcast}
    \IF{$w_p^k>w_q^k$ (resp. $w_p^k<w_q^k$) appears more than $2f$ times} \label{line count}
        \STATE Accept $w_p^k>w_q^k$ (resp. $w_p^k<w_q^k$) \label{codeline accept wp>wq}
    \ENDIF
\ENDFOR
\FOR{\textbf{all} $k \in [m]$}
    \STATE $\mathcal{S}^k_i$ $\gets$ Topological sort all accepted relationships
    \STATE $N_i^k \gets$ Length of $\mathcal{S}^k_i$
    \STATE $\mathbb{C}_i^k$ $\gets$ Participants define sorted $(f+1)$-th to $(N_i^k-f)$-th as Contributors.
\ENDFOR
\ENDFOR
\end{algorithmic}
\end{algorithm}

\subsection{Stage 4: Aggregation and Verification}\label{subsec: stage 3}

In this stage, $P_i$ generates a private one-time pad \cite{bonawitz2017practical}, and applies it to mask the selected contributing coordinates of its local model before transmitting them via BFT broadcast. After receiving the masked coordinates from other participants, each $P_i$ adds the masked values coordinate-wise and verifies authenticity leveraging the commitment scheme. Finally, each $P_i$ computes the global model if the summations of the masked values are authenticated. $P_i$ can also identify the Byzantine adversaries if the summations are not authenticated using \textit{blame}.

\begin{algorithm}[h!]
\caption{Stage 4: Aggregation and Verification}
\label{Participant Only: algorithm sum}
\begin{algorithmic}[1]
\FORALL{$P_i \in \mathbb{P}$ in parallel}
    \FORALL{$P_j \in \mathbb{P} \setminus \{P_i\}$}
        \STATE $a_{ij}=s^1_{ij}+s^1_{ji}$; $r^{a}_{ij}=r^{s1}_{ij}+r^{s1}_{ij}$
        \STATE $b_{ij}=s^2_{ij}+s^2_{ji}$; $r^{b}_{ij}=r^{s2}_{ij}+r^{s2}_{ij}$
        \STATE $\mathcal{A}_{ij}=\mathcal{C}(a_{ij},r^{a}_{ij})$; $\mathcal{B}_{ij}=\mathcal{C}(b_{ij},r^{b}_{ij})$
    \ENDFOR
    \STATE $\mathcal{A}_{i}, \mathcal{B}_{i} \gets \text{Pack all } \mathcal{A}_{ij} \text{ and } \mathcal{B}_{ij}$
    \FORALL{$k \in [m]$}
        \IF{$P_i \in \mathbb{C}_i^k$}
            \STATE $\bar{w}_i^k=\left(w_i^k-\sum_{j \in \mathbb{C}^k_{-i}} a^k_{i j}+\sum_{j \in \mathbb{C}^k_{+i}} a^k_{i j}\right) \bmod \beta$
            \STATE $\bar{r}_i^k=\left(r_i^k-\sum_{j \in \mathbb{C}^k_{-i}} b^k_{i j}+\sum_{j \in \mathbb{C}^k_{+i}} b^k_{i j}\right) \bmod \beta$
        \ENDIF
    \ENDFOR
    \STATE \texttt{BFT-Broadcast}($\bar{w}_i$, $\bar{r}_i$, $\mathcal{A}_{i}$, $\mathcal{B}_{i}$)
    \FORALL{$k \in [m]$}
        \STATE $\bar{w}^k=\sum_{P_j \in \mathbb{C}_i^k} \bar{w}_j^k  \bmod \beta$
        \STATE $\bar{r}^k=\sum_{P_j \in \mathbb{C}_i^k} \bar{r}_j^k  \bmod \beta$
        \IF{$\mathcal{C}(\bar{w}^k,\bar{r}^k) = \prod_{P_j \in \mathbb{C}_i^k} \mathcal{W}_j^k$}
            \STATE $w^k=\frac{\bar{w}^k}{N_i^k-2f}$
        \ENDIF
    \ENDFOR
\ENDFOR
\end{algorithmic}
\end{algorithm}

\begin{algorithm}[h!]
\caption{Blame}
\label{Participant Only: algorithm blame}
\begin{algorithmic}[1]
    \IF{$\mathcal{C}(\bar{w}^k,\bar{r}_j^k) \neq \prod_{P_j \in \mathbb{C}_i^k} \mathcal{W}_j^k$}
        \FORALL{$P_j \in \mathbb{C}_i^k$}
            \IF{$\mathcal{C}(\bar{w}_j^k,\bar{r}^k) \neq \mathcal{W}_j^k \prod_{P_h \in \mathbb{C}^k_{-i}} \mathcal{A}_{ih}^k \prod_{P_h \in \mathbb{C}^k_{+i}} (\mathcal{A}_{ih}^k)^{-1}$}
                \STATE Flag $P_j$ as Byzantine adversary
            \ENDIF
        \ENDFOR
    \ENDIF
\end{algorithmic}
\end{algorithm}
\section{Theoretical Analysis of Brave}
\label{sec: Byzantine Resilience Analysis}
In this section, we theoretically analyze the information-theoretic privacy and Byzantine resilience of \texttt{Brave}.

\subsection{Proof of Information-Theoretic Privacy}

We show that when the passive adversaries are not colluding, \texttt{Brave} guarantees P2P FL to be privacy-preserving as follows.
\begin{theorem}[Information-theoretic Privacy]
    Consider a P2P FL in the presence of passive adversaries who are not colluding. Our proposed protocol, \texttt{Brave}, guarantees information-theoretic privacy of the participants' local models.
\end{theorem}
\begin{proof}[Proof Sketch]
The proof can be divided into two parts.
First, Lemma \ref{lemma Pedersen commitment} shows the information-theoretically hiding property of the commitment, i.e., the commitment itself contains no information about the local model. Second, during the execution of the protocol, the local model is always masked with some random vectors $v_{ij}$ in Algorithm \ref{Participant Only: algorithm aggregator comparison and sorting} and $\sum_{j \in \mathbb{C}^k_{+i}} a^k_{i j}-\sum_{j \in \mathbb{C}^k_{-i}} a^k_{i j}$ in Algorithm \ref{Participant Only: algorithm sum}, which are also information-theoretically hiding by our protocol design.
Therefore, the masked local model is independent of the local model.
\end{proof}





\subsection{Proof of Byzantine Resilience}

In what follows, we establish the Byzantine resilience of P2P FL after applying \texttt{Brave} if $N>3f+2$. 
We first show that using \texttt{Brave}, the messages broadcast by benign participants preserve the ordering of claimed local models.
The preservation of order prevents Byzantine adversaries from forging the results of sorting by manipulating their claimed local models.

\begin{lemma}[Local Non-Forgeability]
	If a benign participant $P_i$ broadcasts $w^k_p>w^k_q$ at coordinate $k$, then the claimed local model of $P_p$ w.r.t. to $\mathcal{W}_p$ is greater than that of $P_q$ at coordinate $k$.
	\label{property non-forgeability}
\end{lemma}
\begin{proof}
Consider three cases when i) both $P_p$ and $P_q$ are benign, ii) either $P_p$ or $P_q$ is Byzantine, iii) both $P_p$ and $P_q$ are Byzantine.
We show the proof of case iii), and similar approaches can be applied to cases i) and ii).

Suppose that in coordinate $k$, $w^k_{p}>w^k_{q}$ holds. To pass the authentication of a benign $P_i$, $P_p$ and $P_q$ must find \textit{a)} another valid pairs $w_{q}'$ or $w_{p}'$ such that $w^{k'}_{q}>w^{k'}_{p}$ but $w^k_{q}<w^k_{p}$, or
\textit{b)} another valid vector $v_{qp}'+v_{pq}'$ such that $v^{k'}_{qp}+v^{k'}_{pq}-v^k_{qp}-v^k_{pq}>w^k_{p}-w^k_{q}$.
However, Lemma \ref{lemma Pedersen commitment} indicates that the Pedersen commitment scheme is computational binding, i.e., a Byzantine participant can not find another value that opens the same commitment.
Therefore, both \textit{a)} and \textit{b)} are infeasible, and the comparison result of $d_{pq}$ and $d_{qp}$ truly reflects the relationship between the claimed local models $w_p$ and $w_q$.
\end{proof}

\textbf{Remark.} In Lemma \ref{property non-forgeability}, we prove that the relationship accepted by benign participants preserves the correct ordering of the claimed local models.

Next, we move from a local perspective to a holistic one. We first show below the Byzantine-resilience of BFT broadcast if $N>3f+2$. We then prove the adversary cannot forge a comparison, as shown in Proposition \ref{property Global Non-Forgeability}. 

\begin{lemma}
If $N>3f+2$, then the messages sent by a participant via BFT broadcast will be received identically by benign participants.
\label{lemma BFT agreement}
\end{lemma}
\begin{proof}
According to \cite{lamport2019byzantine}, consensus can be reached if $N > 3f$ without authentication under synchronous network assumption, and identical information can be received. Thus the lemma follows given $N>3f+2>3f$.
\end{proof}

\begin{proposition}[Global Non-Forgeability]
\label{property Global Non-Forgeability}
	If $w^k_p>w^k_q$ is accepted by a benign participant $P_i$ using Algorithm \ref{Participant Only: algorithm aggregator comparison and sorting}, then the claimed local model of $P_p$ w.r.t. to $\mathcal{W}_p$ is greater than that of $P_q$ at coordinate $k$ given $N>3f+2$.
\end{proposition}
\begin{proof}
If $P_p$ and $P_q$ are both benign, the above argument trivially holds. Then we prove by contradiction to show that the Byzantine adversary cannot forge the result. Assume $w^k_p > w^k_q$ is true, and a benign participant receives a wrong relationship $w^k_p < w^k_q$ for more than $f$ times. Recall that there are at most $f$ faulty participants, so there is at least one benign participant that also broadcast $w^k_p < w^k_q$, which contradicts Lemma \ref{property Global Non-Forgeability}.
\end{proof}

We next prove the liveness property of \texttt{Brave}, i.e., \texttt{Brave} allows the benign participants to accept the true relationships
between their local models, regardless of the adversaries' behaviors during training.


\begin{proposition}[Liveness]
\label{property benign closedness}
Suppose $P_p$ and $P_q$ are benign participants and assume w.l.o.g. $w^k_p>w^k_q$. If $N>3f+2$ holds, then Algorithm \ref{Participant Only: algorithm aggregator comparison and sorting} will guarantee that $w^k_p>w^k_q$ is accepted by all benign participants.
\end{proposition}
\begin{proof}
Since $P_p$ and $P_q$ are benign, they will follow the protocol and send $m_{pq}$ and $m_{pq}$ that can be authenticated to other participants. As $N>3f+2$, apart from $P_p$ and $P_q$, there are at least $(N-2)-f>2f$ benign nodes that will respond with $w^k_p>w^k_q$. Therefore, all benign participants would receive more than $2f$ times $w^k_p>w^k_q$ and accept $w^k_p>w^k_q$.
\end{proof}

Given the aforementioned guarantees provided by \texttt{Brave}, we show that the benign $P_i$ and $P_j$ can reach an agreement on relationship $\bowtie$ during Stage 2.

\begin{proposition}[Agreement on Pairwise Relationships]
	The accepted relationships of all benign participants are identical given $N>3f+2$.
	\label{property Agreement on Pairwise Comparison Result}
\end{proposition}
\begin{proof}
	This property is guaranteed by the BFT broadcast given $N>3f$, as shown in Lemma \ref{lemma BFT agreement}.
\end{proof}

With Proposition \ref{property Global Non-Forgeability} and \ref{property Agreement on Pairwise Comparison Result}, we can now conclude that the relationships among benign participants are \textit{correct} and \textit{identical}. Therefore, an agreement on sorting can be reached.

\begin{corollary}[Agreement on Sorting]
	Given any benign participants $P_i$ and $P_j$, if $N>3f+2$ holds, then \texttt{Brave} guarantees the sorting results $\mathcal{S}^k_i=\mathcal{S}^k_j$ for all $k$, i.e., the sorting of all benign participants is identical.
	\label{property agreement on ranking}
\end{corollary}

Based on Corollary \ref{property agreement on ranking}, we show that \texttt{Brave} guarantees all benign participants to obtain an identical global model $w$.

\begin{proposition}[Agreement on Global Model]
The global model $w$ of all benign participants is identical given $N>3f+2$.
\label{property agreement on global model}    
\end{proposition}
\begin{proof}
    Given Corollary \ref{property agreement on ranking} and Lemma \ref{lemma BFT agreement}, $\mathcal{S}_i^k=\mathcal{S}_j^k$, $\bar{w}_i^k=\bar{w}_j^k$, and $\bar{r}_i^k=\bar{r}_j^k$ for all benign participants $P_i$ and $P_j$. Thus $\bar{w}^k$ calculated in Line 16 is identical among benign participants, leading to identical verification result in Line 18 of Algorithm \ref{Participant Only: algorithm sum}. Therefore, the global model $w$ is also identical.
\end{proof}


We conclude this section by proving that the global model obtained by the benign participants deviates from the optimal one by a bounded distance, i.e., satisfies $\epsilon$-convergence property.
\begin{theorem}
     Assume that local model $w_i(t)$ belongs to a bounded set and is $z$-sub-exponential for all $i$ and $t$.
     The global model $w(T)$ obtained by applying \texttt{Brave} to P2P FL satisfies $\|w(T)-w(T)^*\|\leq \epsilon$ with probability at least $\zeta$, where
     \begin{align*}
         \zeta &= 1-2m\exp\{-(N-f)\min_{i}|\mathcal{D}_i|\min\{\frac{\tau}{2z},\frac{\tau^2}{2z^2}\}\}\\
         &-2(N-f)m\exp\{\min_i|\mathcal{D}_i|\min\{\frac{\delta}{2z},\frac{\delta^2}{2z^2}\}\},
     \end{align*}
     $\epsilon = \sqrt{m}\frac{\tau+3f \delta/N}{1-2f/N}$, and $w(T)^*$ is the optimal global model obtained using Eqn. \eqref{equation weight update} and \eqref{equation directly average weight} when P2P FL terminates at iteration $T$.
     
\end{theorem}
\begin{proof}
Using Corollary \ref{property agreement on ranking}, we have that benign participants $P_i$ and $P_j$ own identical $\mathcal{S}_i^k=\mathcal{S}_j^k$ for all benign participants $i,j$. 
They further hold identical global model $w(t)$ for all iteration $t$. 
Consider an arbitrary benign participant. 
If $w(T)$ is $z$-sub-exponential, then using Lemma 3 of \cite{yin2018byzantine} yields that
\begin{equation*}
    |w(T)^k - w(T)^{k^*}|\leq \frac{\tau+3f \delta/N}{1-2f/N}
\end{equation*}
holds for each coordinate $k$ with probability at least $1-2\exp\{-(N-f)\min_{i}|\mathcal{D}_i|\min\{\frac{\tau}{2z},\frac{\tau^2}{2z^2}\}\}-2(N-f)\exp\{\min_i|\mathcal{D}_i|\min\{\frac{\delta}{2z},\frac{\delta^2}{2z^2}\}\}$.
By using the definitions of $\|\cdot\|$ and union bound over all coordinate $k$, we have that
\begin{equation*}
    \|w(T) - w(T)^{*}\|\leq \sqrt{m}\frac{\tau+3f \delta/N}{1-2f/N}
\end{equation*}
holds with probability at least $\zeta$ for any $\tau,\delta\geq 0$.
\end{proof}


\end{document}